\documentclass[a4paper,11pt]{article}
\usepackage[utf8]{inputenc}

\usepackage{color}

\usepackage{amssymb}
\usepackage{amsmath}
\usepackage{amsfonts}
\usepackage{amsthm}
\usepackage{url}

\usepackage[dvips]{epsfig}
\usepackage[dvips]{graphics}

\usepackage{algorithm}
\usepackage{algorithmic}

\usepackage[all,knot,poly,arc]{xy}

\newcommand{\NN}{{\Bbb N}}
\newcommand{\RR}{{\Bbb R}}

\newcommand{{\uk}}{\mbox{$\underline{k}$}}

\def\nod(#1,#2){\put(#1,#2){\circle*{.125}}
\put(#1,#2){\makebox(0,0.5){{\small$#2$}}}}%
\def\rod(#1,#2){\put(#1,#2){\circle*{.2}}}
\def\NOD(#1,#2)#3{\put(#1,#2){\circle*{.2}}\put(#1,#2){\makebox(0,0.8){{\small$#3$}}}}

\def\EXX{{\hfill{$\diamondsuit$}}}

\newcounter{exampleNo}

\newtheorem{theorem}{Theorem}[section]

\newtheorem{proposition}[theorem]{Proposition}

\newenvironment{example}[1][Example \arabic{exampleNo}.]{\begin{trivlist}\refstepcounter{exampleNo}
\item[\hskip \labelsep {\bfseries #1}]}{\end{trivlist}}

\title{Computations in Stochastic Acceptors}
\author{Karl-Heinz Zimmermann\footnote{Email: k.zimmermann@tuhh.de}\\
Department of Electrical Engineering, Computer Science, Mathematics\\
Hamburg University of Technology\\
21071 Hamburg, Germany}

\begin{document}
\maketitle
\begin{abstract}
Machine learning provides algorithms that can learn from data and make inferences or predictions on data.
Stochastic acceptors or probabilistic automata are stochastic automata without output that can model components in machine learning scenarios.
In this paper, we provide dynamic programming algorithms for the computation of input marginals and the acceptance probabilities in stochastic acceptors. 
Furthermore, we specify an algorithm for the parameter estimation of the conditional probabilities 
using the expectation-maximiza\-tion technique and a more efficient implementation related to the Baum-Welch algorithm.
\end{abstract}
\medskip

\mbox{\bf AMS Subject Classification:} 68Q70, 68T05 
\medskip

\mbox{\bf Keywords:} Probabilistic automaton, dynamic programming, parameter estimation, EM algorithm, Baum-Welch algorithm

\section{Introduction}

The theory of discrete stochastic systems has been first studied by Shannon~\cite{shannon} and von Neumann~\cite{neumann}.
Shannon has considered memory-less communication channels and their generalization by introducing states,
while von Neumann has investigated the synthesis of reliable systems from unreliable components.
The seminal research work of Rabin and Scott~\cite{rscott} about deterministic finite-state automata has led to two generalizations.
First, the generalization of transition functions to conditional distributions studied by Carlyle~\cite{carl} and Starke~\cite{starke}.
Second, the generalization of regular sets by introducing stochastic acceptors as described by Rabin~\cite{rabin}. 

A stochastic acceptor or probabilistic automaton is a stochastic automaton without output~\cite{claus, salomaa, zim-SA}.
It generalizes the nondeterministic finite automaton by involving the probability of transition from one state to another and in this way generalizes the concept of Markov chain.
The languages accepted by stochastic acceptors are called stochastic languages.
The class of stochastic languages is uncountable and includes the regular languages as a proper subclass. 

Stochastic automata have widespread use in the modeling of stochastic systems such as in traffic theory and in spoken language understanding 
for the recognition and interpretation of speech signals~\cite{claus, rab-juang, ricc}.
They can be used as building blocks in situations of machine learning where detailed mathematical description is missing and feature management is noisy.
The arrangement of stochastic automata in the form of teams or hierarchies could lead to solutions of complex inference problems~\cite{mal}.

Stochastic acceptors have been generalized to a quantum analog, the quantum finite automaton~\cite{kondacs}.  
The latter are linked to quantum computers as stochastic acceptors are connected to conventional computers.

In this paper, we provide dynamic programming algorithms for the computation of input marginals and the acceptance probabilities in a stochastic acceptor. 
Moreover, we specify an algorithm for the parameter estimation of the conditional probabilities using the expectation-maximization technique and a variant of the Baum-Welch algorithm.
The text is to a large extent self-contained and also suitable to non-experts in this field.

\section{Mathematical Preliminaries} 
A {\em stochastic acceptor\/}  (SA)~\cite{claus, rabin, salomaa} is a quintuple $A = (S,\Sigma, P, \pi, f)$, where
$S$ is a nonempty finite set of {\em states},
$\Sigma$ is an alphabet of {\em input symbols},
$P$ is a collection $\{P(a)\mid a\in\Sigma\}$ of stochastic $n\times n$ matrices, where $n$ is the number of states, 
$\pi$ is the initial distribution of the states written as row vector, and $f$ is a binary column vector of length $n$ called {\em final state vector}.

Let $S=\{s_1,\ldots,s_n\}$ be the state set.
Then the final state vector is $f=(f_1,\ldots,f_n)^t$ and $F = \{s_i\mid f_i=1\}$ is the {\em final state set}.
Moreover, the matrices $P(a) = (p_{ij}(a))$ with $a\in\Sigma$ are transition probability matrices, where the $(i,j)$th entry $p_{ij}(a)= p(s_j\mid a,s_i)$ is the 
conditional probability of transition from state $s_i$ to state $s_j$ when the symbol $a$ is read, $1\leq i,j\leq n$.
Thus for each symbol $a\in\Sigma$ and each state $s\in S$,
\begin{eqnarray}\label{e-SA-marg0}
\sum_{s'\in S} p(s'\mid a,s) =1.
\end{eqnarray}
Given a conditional probability distribution $p(\cdot\mid a,s)$ on $\Sigma\times S$,  
a probability distribution $\hat p$ on $\Sigma^*\times S$ can be defined recursively as follows.
\begin{itemize}
\item For each $s,s'\in S$,
\begin{eqnarray}\label{e-SA-phat1}
\hat p (s'\mid \epsilon,s) = \left\{ \begin{array}{ll} 1 & \mbox{if } s=s',\\ 0 & \mbox{if } s\ne s',  \end{array} \right.
\end{eqnarray}
where $\epsilon$ denotes the empty word in $\Sigma^*$.
\item For all $s,s'\in S$, $a\in \Sigma$, and $x\in\Sigma^*$, 
\begin{eqnarray}\label{e-SA-phat3}
\hat p (s'\mid xa,s) = \sum_{t\in S} \hat p(t\mid x,s)\cdot p(s'\mid a,t).
\end{eqnarray}
\end{itemize}
Then $\hat p(\cdot\mid x,s)$ is a conditional probability distribution on $\Sigma^*\times S$ and so we have
\begin{eqnarray}\label{e-SA-phat4}
\sum_{s'\in S} p(s'\mid x,s) =1,\quad x\in\Sigma^*, s\in S.
\end{eqnarray}
Note that the measures $p$ and $\hat p$ coincide on the set $S\times\Sigma\times S$ if we put $x=\epsilon$ in~(\ref{e-SA-phat3}).
Therefore, we write $p$ instead of $\hat p$.

A stochastic acceptor works serially and synchronously.
It reads an input word symbol by symbol and after reading an input symbol it transits into another state.
In particular, if the automaton starts in state $s$ and reads the word $x$, then
with probability $p(s'\mid x,s)$ it will end in state $s'$ taking all intermediate states into account.

\begin{proposition}\label{p-SA-prob0}
For all $x,x'\in\Sigma^*$ and $s,s'\in S$, 
\begin{eqnarray}
p(s'\mid xx',s) = \sum_{t\in S} p(t\mid x,s) \cdot p(s'\mid x',t). 
\end{eqnarray}
\end{proposition}
This result can be described by probability matrices.
To this end, for the empty word $\epsilon\in\Sigma^*$ put
\begin{eqnarray}
P(\epsilon) = I_n,
\end{eqnarray}
where $I_n$ is the $n\times n$ unit matrix. 
Furthermore, if $a\in\Sigma$ and $x\in\Sigma^*$, then by~(\ref{e-SA-phat3}) 
\begin{eqnarray}
P(xa) = P(x)\cdot P(a).
\end{eqnarray}
By Prop.~\ref{p-SA-prob0} and the associativity of matrix multiplication, we obtain the following
\begin{proposition}\label{p-SA-prob1}
For all $x,x'\in\Sigma^*$,
\begin{eqnarray}
P(xx') = P(x)\cdot P(x').
\end{eqnarray}
\end{proposition}
\noindent
It follows by induction that if $x=x_1\ldots x_k\in\Sigma^*$, then
\begin{eqnarray}\label{e-SA-prob2}
P(x) = P(x_1)\cdots P(x_k).
\end{eqnarray}

Let $A= (S,\Sigma, P, \pi,f)$ be a stochastic acceptor and let $\lambda$ be a real number with $0\leq\lambda\leq 1$.
The set
\begin{eqnarray}\label{e-L}
L_{A,\lambda} = \{x\in\Sigma\mid \pi P(x) f>\lambda\}
\end{eqnarray}
is the language of $A$ w.r.t.\ $\lambda$, and $\lambda$ is the {\em cut point\/} of $L_{A,\lambda}$.

\begin{example}
Let $p\geq 2$ be an integer.
Consider the $p$-adic stochastic acceptor $A = (\{s_1,s_2\},\{0,\ldots,p-1\},P, \pi,f)$ with
$$P(a)= \left(\begin{array}{cc}1-\frac{a}{p} &\frac{a}{p}\\1-\frac{a+1}{p}& \frac{a+1}{p}\end{array}\right),
\; 0\leq a\leq p-1,
\; \pi=(1,0), 
\;\mbox{and}\; f=\left(\begin{array}{c}0\\1\end{array}\right).$$ 
See Fig.~\ref{fi-SA-A}.
Each word $x=x_1\ldots x_k\in\{0,\ldots,p-1\}^*$ can be assigned the real number whose $p$-adic representation is $0.x_k\ldots x_1$.
For each cut point $\lambda$,  the accepted language is
$$L_{A,\lambda} = \{x_1\ldots x_k\in\{0,\ldots,p-1\}^* \mid 0.x_k\ldots x_1>\lambda\}.$$
Note that the language $L_{A,\lambda}$ is regular if and only if the cut point $\lambda$ is rational~\cite{claus, rabin, rscott}.
\EXX
\end{example}
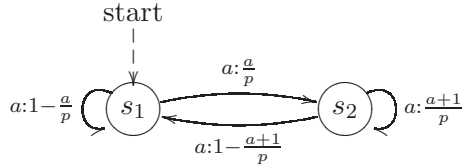
\begin{figure}[hbt]
\begin{center}
\mbox{$
\xymatrix{
\txt{start}\ar@{-->}[d] && \\
*++[o][F-]{s_1} 
\ar@(ul,dl)[]_{a:1-\frac{a}{p}}
\ar@/^/[rr]^{a:\frac{a}{p}}
&&
*++[o][F-]{s_2} 
\ar@/^/[ll]^{a:1-\frac{a+1}{p}}
\ar@(ur,dr)[]^{a:\frac{a+1}{p}}
}
$}
\end{center}
\caption{State diagram of $A$.}\label{fi-SA-A}
\end{figure}

For each input word $x\in\Sigma^*$, the stochastic matrix $P(x)$ can be viewed as generating a discrete-time Markov chain.
Thus the behavior of a stochastic automaton is an interleaving of Markov chains each of which corresponding to a single input symbol.
\section{Input Marginals and Acceptance Probabilities}
The input marginals and the acceptance probabilities can be computed by the technique of dynamic programming~\cite{bellman} using sum-product decomposition.

To see this, let $A=(S,\Sigma,P,\pi,f)$ be a stochastic acceptor with $l$-element state set $S$ and $l'$-element input set $\Sigma$.
A stochastic acceptor can be viewed as a belief network. 
To this end, let $n\geq 1$ be an integer.
Let $X_1,\ldots,X_n$ be random variables with common state set $\Sigma$ and
let $S_1,\ldots,S_{n+1}$ be random variables with common state set $S$.
The stochastic acceptor can be described for inputs of length $n$ by the belief network~\cite{barber,koski,zim-SA} as shown in Fig.~\ref{f-SA-bn}.
Then the corresponding joint probability distribution factoring according to the network is given by 
\begin{eqnarray}
p_{X,S} = p_{S_1}p_{X_1}p_{S_2|X_1,S_1} p_{X_2} p_{S_3|X_2,S_2}\cdots p_{X_n} p_{S_{n+1}|X_n,S_n}.
\end{eqnarray}
We assume for simplicity that the initial distributions $p_{X_i}$ are uniform; 
i.e., $p_{X_i}(x)=\frac{1}{l'}$ for all $x\in \Sigma$ and $1\leq i\leq n$.
Moreover, the network is assumed to be {\em homogeneous\/} in the sense that the conditional distributions 
$p_{S_{i+1}|X_i,S_i}$ are independent of the index $i$, $1\leq i\leq n$.
Therefore, we put
\begin{eqnarray}
\theta_{s';a,s} = p_{S_{i+1}|X_i,S_i}(s'\mid a,s), \quad s,s'\in S,\;a\in\Sigma, \;1\leq i\leq n.
\end{eqnarray}
It follows that the joint probability distribution has the form
\begin{eqnarray}
p_{X,S}(x_1,\ldots,x_n,s_1,\ldots,s_{n+1}) = \frac{1}{l'^n} \pi_{s_1} \theta_{s_2;x_1,s_1} \cdots \theta_{s_{n+1};x_n,s_n}.
\end{eqnarray}
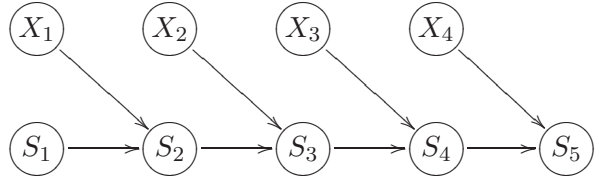
\begin{figure}[hbt]
\begin{center}
\mbox{$
\xymatrix{
*++[o][F-]{X_1} \ar@{->}[dr] & *++[o][F-]{X_2} \ar@{->}[dr] & *++[o][F-]{X_3} \ar@{->}[dr] & *++[o][F-]{X_4} \ar@{->}[rd] \\
*++[o][F-]{S_1} \ar@{->}[r]  & *++[o][F-]{S_2 } \ar@{->}[r] & *++[o][F-]{S_3}  \ar@{->}[r] & *++[o][F-]{S_4} \ar@{->}[r] & *++[o][F-]{S_5} \\
}
$}
\end{center}
\caption{Belief network of stochastic acceptor with $n=4$.}\label{f-SA-bn}
\end{figure}

The probability of an input sequence $x=x_1\ldots x_n\in\Sigma^n$ is given by the marginal distribution
\begin{eqnarray}
p_{X}(x) = \sum_{s\in S^{n+1}} p_{X,S}(x,s).
\end{eqnarray}
The corresponding sum-product decomposition yields
\begin{eqnarray}\label{e-SA-spd}
\lefteqn{ p_{X}(x) = }\\
&&
\frac{1}{l'^n} 
\sum_{s_{n+1}\in S} 
\left(
\sum_{s_n\in S} \theta_{s_{n+1}; x_n,s_n}
\left(
\ldots 
\left(
\sum_{s_1\in S} \theta_{s_2; x_1,s_1}\cdot \pi_{s_1}\right)\ldots\right)\right).\nonumber
\end{eqnarray}
According to this decomposition, the marginal probability $p_X(x)$ can be calculated by using an $n\times l$ table $M$:
\begin{eqnarray}\label{e-M-1}
M[0,s] &:=& 
\pi_{s},\quad s\in S,\nonumber\\
M[k,s] &:=& \sum_{s'\in S} \left(\theta_{s;x_k,s'} \cdot M[k-1,s']\right), \quad s\in S,\,1\leq k\leq n,\\
p_X(x) &:=& \frac{1}{l'^n} \sum_{s\in S} M[n,s].\nonumber
\end{eqnarray}
The time complexity of this algorithm is $O(l^2n)$, since the table $M$ has size $O(ln)$ and each table entry is computed in $O(l)$ steps.
The marginal probabilities $p_X(x)$ will be used in the EM and BM algorithms later on.

On the other hand, the acceptance probability of an input sequence $x=x_1\ldots x_n\in\Sigma^n$ is given by the sum-product decomposition
\begin{eqnarray}\label{e-M-2}
\lefteqn{\pi P(x) f = }\\
&&
\sum_{s_{n+1}\in S} 
f_{s_{n+1}}
\left(
\sum_{s_n\in S} \theta_{s_{n+1}; x_n,s_n}
\left(
\ldots 
\left(
\sum_{s_1\in S} \theta_{s_2; x_1,s_1}\cdot \pi_{s_1} \right)\ldots\right)\right).\nonumber
\end{eqnarray}
This decomposition can be used to compute the acceptance probability by using an $n\times l$ table $M$:
\begin{eqnarray}
M[0,s] &:=& \pi_s,\quad s\in S,\nonumber\\
M[k,s] &:=& \sum_{s'\in S} \left(\theta_{s;x_k,s'} \cdot M[k-1,s']\right), \quad s\in S,\,1\leq k\leq n,\\
\pi P(x) f &:=& \sum_{s\in S} f_s\cdot M[n,s].\nonumber
\end{eqnarray}
Similarly, the time complexity of this algorithm is $O(l^2n)$, since the table $M$ has size $O(ln)$ and each table entry is computed in $O(l)$ steps.

\section{Parameter Estimation}

The objective is to estimate the conditional probabilities of a stochastic acceptor by using sample data.
For this, the stochastic acceptor is viewed as a belief network as described in the previous section.
For this, let $A=(S,\Sigma,P,\pi,f)$ be a stochastic acceptor with $l=|S|$ and $l'=|\Sigma|$, and let $n\geq 1$.
Take the parameter set
\begin{eqnarray}
\Theta = \{\theta=(\theta_{s';a,s})\mid \theta_{s';a,s}\geq 0,\sum_{s'}\theta_{s';a,s} = 1\}.
\end{eqnarray}
where
\begin{eqnarray}
\theta_{s';a,s} = p(s'\mid a,s),\quad a\in\Sigma,\; s,s'\in S.
\end{eqnarray}

The aim is to estimate these probabilities by making use of a sample set.
For this, assume that there is a collection $D= (d_1,\ldots,d_N)$ of $N$ independent samples called {\em database}, where 
$d_r= (x_r,s_r) \in\Sigma^n\times S^{n+1}$ denotes the $r$-th sample, $1\leq r\leq N$.
For simplicity, suppose the initial distributions $p_{X_i}$ are uniform as before, $1\leq i\leq n$.
Then the joint probability of the sample $d_r=(x_r,s_r)$ depending on the parameters is given by
\begin{eqnarray}
p_{X,S|\Theta} (d_r \mid \theta) 
= \frac{1}{l'^n} \pi_{s_1}\prod_{i=1}^n \theta_{s_{r,i+1}; x_{r,i}, s_{r,i}}.
\end{eqnarray}
Thus the likelihood function $L=L_{X,S}$ is given by
\begin{eqnarray}
L(\theta) 
&=& \prod_{r=1}^N p_{X,S|\Theta} (d_r \mid \theta)= \prod_{(x,s)} p_{X,S|\Theta} (x,s\mid \theta)^{u_{x,s}}, 
\end{eqnarray}
where $u_{x,s}$ is the number of times the input-state pair $(x,s)$ is observed in the sample set.
Therefore, we have
\begin{eqnarray}
\sum_{(x,s)} u_{x,s}  = N.
\end{eqnarray}

Let $v_{s';a,s}$ be the number of times the parameter $\theta_{s';a,s}$ occurs in the likelihood function $L(\theta)$.
Then the likelihood function can be written (up to a constant) as
\begin{eqnarray}
L(\theta) = \prod_{a\in\Sigma} \prod_{s,s'\in S} \theta_{s';a,s}^{v_{s';a,s}}.
\end{eqnarray}
The corresponding log-likelihood function $\ell=\ell_{X,S}$ is
\begin{eqnarray}
\ell(\theta) = \log L(\theta) = \sum_{a\in\Sigma} \sum_{s,s'\in S} v_{s';a,s} \theta_{s';a,s}.
\end{eqnarray}

The data $v=(v_{s';a,s})$ form the sufficient statistic of the model.
These data can be obtained from the given data $u=(u_{x,s})$ by the linear transformation
\begin{eqnarray}\label{e-SA-A}
v = B_{l,l'}\cdot u,
\end{eqnarray}
where $B=B_{l,l'}$ is an integral matrix with $d=l^2l'$ rows labeled by the triples $(s';a,s)$ with $a\in\Sigma$ and $s,s'\in S$.
Moreover, the matrix has $m=l'^nl^{n+1}$ columns labeled by the pairs $(x,s)\in\Sigma^n\times S^{n+1}$.
The matrix has entry~$k$ in row $(s';a,s)$ and column $(x,s)$ if the parameter $\theta_{s';a,s}$ occurs $k$ times in $p_{X,S|\Theta}(x,s)$. 
Note that the matrix has column sum $n$, since the quantity $p_{X,S|\Theta}(x,s)$ has $n$ factors.

\begin{example}
Consider the 2-adic stochastic acceptor $A$ with state set $S =\{a,b\}$ and input set $\Sigma = \{0,1\}$, and let $n=2$.
The associated $8\times 32$ matrix $B=B_{2,2}$ is as follows,
{\tiny 
$$\bordermatrix{
~       & 00,aaa  & 00,aab    & 00,aba    & 00,abb    & 00,baa    & 00,bab    & 00,bba    & 00,bbb &\ldots & 11,bbb \cr
a;0,a & 2         & 1         & 0         & 0         & 1         &  0        & 0         & 0      &       & 0  \cr
b;0,a & 0         & 1         & 1         & 1         & 0         &  1        & 0         & 0      &       & 0  \cr
a;1,a & 0         & 0         & 0         & 0         & 0         &  0        & 0         & 0      &       & 0  \cr
b;1,a & 0         & 0         & 0         & 0         & 0         &  0        & 0         & 0      &       & 0  \cr
a;0,b & 0         & 0         & 1         & 0         & 1         &  1        & 1         & 0      &       & 0  \cr
b;0,b & 0         & 0         & 0         & 1         & 0         &  0        & 1         & 2      &       & 0  \cr
a;1,b & 0         & 0         & 0         & 0         & 0         &  0        & 0         & 0      &       & 0  \cr
b;1,b & 0         & 0         & 0         & 0         & 0         &  0        & 0         & 0      &       & 2  \cr
}. $$
}
\EXX
\end{example}

\begin{proposition}\label{p-SA-fo-l}
The maximum likelihood estimate of the likelihood function $L(\theta)$ is given by
\begin{eqnarray}
\hat\theta_{s';a,s} = \frac{v_{s';a,s}}{\sum_{s''\in S} v_{s'';a,s}},
\quad a\in\Sigma, \,s,s'\in S.
\end{eqnarray}
\end{proposition}
\begin{proof}
Let $S=\{s_1,\ldots,s_l\}$ and $\Sigma=\{a_1,\ldots,a_{l'}\}$.
For each input-state pair $(a_i,s_j)$, $1\leq i\leq l'$, $1\leq j\leq l$,  we have 
$$\sum_{m=1}^l \theta_{s_m;a_is_j}= 1.$$
The parameters $\theta_{s_m;a_is_j}$ with $1\leq m\leq l$  appear in the log-likelihood function $\ell(\theta)$ as the partial sum
$$\ell_{i,j} = \sum_{m=1}^l v_{a_m;a_i,s_j}\log(\theta_{a_m;a_i,s_j}).$$
Using $\theta_{s_l;a_i,s_j} = 1- \sum_{s_m\ne s_l}\theta_{s_m;a_i,s_j}$,
the partial derivative of $\ell_{i,j}$ with respect to $\theta_{s_m;a_i,s_j}$ becomes
$$\frac{\partial \ell_{i,j}}{\partial \theta_{s_m;a_i,s_j}} 
= \frac{v_{s_m;a_i,s_j}}{\theta_{s_m;a_i,s_j}} - \frac{v_{s_l;a_i,s_j}}{1- \sum_{s_m\ne s_l}\theta_{s_m;a_i,s_j}}.$$
Equating this expression to~0 gives $\hat\theta_{s_m;a_i,s_j}$ as claimed.
Thus the vector $\hat \theta = (\hat\theta_{s_m;a_i,s_j})$ is a critial point of the likelihood function.

Claim that this point maximizes the likelihood function; the proof idea goes back to Koski et al.~\cite{koski}.
Indeed, let $H(\theta)=-\sum_{i=1}^n \log\theta_i$ denote the entropy of a probability distribution $\theta=(\theta_1,\ldots,\theta_n)$ and
let $D(\theta\|\theta') = \sum_{i=1}^n\theta_i\log\left(\frac{\theta_i}{\theta'_i}\right)$ denote the Kullback-Leibler measure between two probability distributions
$\theta=(\theta_1,\ldots,\theta_n)$ and $\theta'=(\theta'_1,\ldots,\theta'_n)$.
Then we have 
\begin{eqnarray*}
\ell (\theta) &=& \sum_{a\in\Sigma}\sum_{s,s'\in S} v_{s';a,s}\log \theta_{s';a,s} \\
&=& \sum_{a\in\Sigma} \sum_{s,s',s''\in S} v_{s'';a,s}  \hat\theta_{s';a,s}\log \theta_{s';a,s} \\
&=& \sum_{a\in\Sigma} \sum_{s\in S} v_{a,s}\left(\sum_{s'\in S} \hat\theta_{s';a,s}\log \theta_{s';a,s}\right) \\
&=& \sum_{a\in\Sigma} \sum_{s\in S} v_{a,s}\left(\sum_{s'\in S} \hat\theta_{s';a,s}\log \hat \theta_{s';a,s} 
- \hat\theta_{s';a,s}\log \frac{\hat \theta_{s';a,s}}{\theta_{s';a,s}}\right) \\
&=&     \sum_{a\in\Sigma}\sum_{s\in S} -v_{a,s} \left(  H(\hat \theta_{a,s}) + D(\hat\theta_{a,s}\|\theta_{a,s})\right), 
\end{eqnarray*}
where $v_{a,s} = \sum_{s''\in S} v_{s'';a,s}$, 
$\theta_{a,s}=(\theta_{s';a,s})$ and $\hat \theta_{a,s}=(\hat \theta_{s';a,s})$ for each input-state pair $(a,s)$.
Since the Kullback-Leibler measure is always non-negative~\cite{koski}, we obtain
\begin{eqnarray*}
\ell(\theta) &=&     \sum_{a\in\Sigma}\sum_{s\in S} -v_{a,s} \left(  H(\hat \theta_{a,s}) + D(\hat\theta_{a,s}\|\theta_{a,s})\right)\\ 
&\leq&     \sum_{a\in\Sigma}\sum_{s\in S} -v_{a,s} H(\hat \theta_{a,s}) \\
&=& \sum_{a\in\Sigma}\sum_{s\in S} v_{a,s}\sum_{s'\in S} \hat\theta_{s';a,s}\log\hat\theta_{s';a,s}\\
&=& \sum_{a\in\Sigma}\sum_{s,s'\in S} v_{s';a,s}\log\hat\theta_{s';a,s}\\
&=&\ell(\hat\theta).
\end{eqnarray*}
This proves the claim and the result follows.
\end{proof}

A stochastic acceptor is an abstract machine with an input interface.
Therefore, suppose the sample data consist only of the input sequences, while the observer has no access to the state sequences.
This problem can be tackled by the expectation-maximization (EM) algorithm.
This is an iterative method to find the maximum posterior estimates of parameters in a statistical model with unobserved latent variables.

The aim is to estimate these probabilities by making use of a sample set.
For this, let $A=(S,\Sigma,P,\pi,f)$ be a stochastic acceptor in the above setting and let $n\geq 1$.
We assume that there is a collection $D= (d_1,\ldots,d_N)$ of $N$ independent samples called {\em database}, where 
$d_r= x_r \in\Sigma^n$ denotes the $r$-th input sample, $1\leq r\leq N$.
Then the probability of the sample $d_r$ depending on the parameters is given by the marginal distribution
\begin{eqnarray}
p_{X|\Theta} (d_r \mid \theta) 
= \sum_{s\in S^{n+1}} p_{X,S|\Theta}(x_r,s\mid\theta).
\end{eqnarray}
The likelihood function $L=L_{X}$ is given by
\begin{eqnarray}
L(\theta) 
&=& \prod_{r=1}^N p_{X|\Theta} (d_r \mid \theta) = \prod_{x} p_{X|\Theta} (x\mid \theta)^{u_{x}}, 
\end{eqnarray}
and the log-likelihood function $\ell=\ell_{X}$ is
\begin{eqnarray}
\ell(\theta) =\log L(\theta) = \sum_{x} u_{x}\log p_{X|\Theta} (x\mid \theta),
\end{eqnarray}
where $u_{x}$ is the number of times the input sequence $x$ is observed in the sample set.
Therefore, we have
\begin{eqnarray}
\sum_{x} u_{x}  = N.
\end{eqnarray}

A version of the EM algorithm for stochastic acceptors is given by Alg.~\ref{a-SA-em}.
Note that in the E-step, the marginal probabilities $p_{X}(x|\theta)$ can be efficiently computed by the sum-product decomposition~(\ref{e-M-1}).
In the M-step, the maximal estimate $\hat \theta $ can be calculated directly by using Prop.~\ref{p-SA-fo-l}.
In the compare step, it can be shown that the inequality $\ell_X(\hat\theta)\geq \ell_X(\theta)$ always holds~\cite{sturm,zim}.
\begin{algorithm}
\caption{EM algorithm for stochastic acceptor}\label{a-SA-em}\index{EM algorithm!stochastic acceptor}
\begin{algorithmic}
\REQUIRE Stochastic acceptor $A=(S,\Sigma,P,\pi,f)$, joint probability function $p_{X,S|\Theta}$, 
parameter space $\Theta\subseteq\RR_{>0}^{l' l(l-1)}$,
integer $n\geq 1$,
observed data $u=(u_{x})\in\NN^{l'^n}$
\ENSURE Maximum likelihood estimate $\theta^*\in\Theta$
\STATE [Init] Threshold $\epsilon > 0$ and parameters $\theta\in\Theta$
\STATE [E-Step] Define matrix $U=(u_{x,s})\in\RR^{l'^n\times l^{n+1}}$ with
$$u_{x,s} = \frac{u_{x}\cdot p_{X,S|\Theta}(x,s|\theta)}{p_{X|\Theta}(x|\theta)},\quad x\in\Sigma^n,\;s\in S^{n+1}$$ 
\STATE [M-Step] Compute solution $\hat \theta\in\Theta$ of the likelihood function $\ell_{X,S}$ using the data set $U=(u_{x,s})$ as in Prop.~\ref{p-SA-fo-l}  
\STATE [Compare] If $\ell_{X}(\hat \theta) - \ell_{X}(\theta)>\epsilon$, set $\theta \leftarrow \hat\theta$ and resume with E-step
\STATE [Output] $\theta^* \leftarrow \hat\theta$
\end{algorithmic}
\end{algorithm}

The structure of stochastic acceptors allows a more efficient implementation of the EM algorithm which amounts to a variant of the Baum-Welch algorithm~\cite{sturm, zim-SA}.
To see this, let $n\geq 1$ be an integer.
Let $u=(u_{x})\in\NN^{l'^n}$ be a data vector, 
where $u_{x}$ is the number of times the input sequence $x\in\Sigma^n$ is observed in the sample set.
The full data vector $U=(u_{x,s})\in\NN^{l'^n\times l^{n+1}}$ is not available, 
where $u_{x,s}$ denotes the number of times the pair $(x,s)\in\Sigma^n\times S^{n+1}$ is observed.
The EM algorithm estimates in the E-step the counts of the full data vector by the quantity
\begin{eqnarray}\label{eq-SA-Eu}
u_{x,s} = \frac{u_{x}\cdot p_{X,S|\Theta}(x,s|\theta)}{p_{X|\Theta}(x|\theta)},\quad x\in\Sigma^n,\;s\in S^{n+1}.
\end{eqnarray}
These counts provide the sufficient statistic $v$ of the model and are used in the M-step to obtain updated parameter values based on
the solution of the maximum likelihood problem in Prop.~\ref{p-SA-fo-l}.
The expected values of the sufficient statistic $v$ can be written in a way that leads to a more efficient implementation of the 
EM algorithm using dynamic programming.

For this, we introduce socalled forward and backward probabilities.
The {\em forward probability} 
\begin{eqnarray}
f_{x,s}(i) = p_{X_1,\ldots,X_i,S_{i+1}}(x_1,\ldots,x_i,s),
\end{eqnarray}
where $s\in S$ and $1\leq i\leq n$,
is the joint probability that the prefix $x_1\ldots x_i$ of the observed input sequence $x\in\Sigma^n$ having length $i$ ends in state~$s$.
For simplicity, assume that the initial distribution of $S_1$ is uniform; i.e., $p_{S_1}(s)=\frac{1}{l}$ for all $s\in S$.
Then we put $f_{x,s}(0)= \frac{1}{l\cdot l'^n}$.

The {\em backward probability}
\begin{eqnarray}
b_{x,s}(i) = p_{X_{i+1},\ldots,X_n|S_{i+1}}(x_{i+1},\ldots,x_n\mid s),
\end{eqnarray}
where $s\in S$ and $0\leq i\leq n-1$,
is the conditional probability that the suffix $x_{i+1}\ldots x_n$ of the observed input sequence $x\in\Sigma^n$ having length $n-i$ starts in state $s$.

The marginal probability $p_{X|\Theta}(x|\theta)$ of the observed input sequence $x\in\Sigma^n$ can be calculated based on the forward probabilities,
\begin{eqnarray}
p_{X|\Theta}(x|\theta) = \sum_{s\in S} f_{x,s}(n) .
\end{eqnarray}
Note that the forward and backward probabilities can be recursively computed.
To see this, consider for the input sequence $x\in \Sigma^n$ the $l\times n$ matrices $F_{x}=(f_{x,s}(i))_{s,i}$ and $B_{x}=(b_{x,s}(i))_{s,i}$  
corresponding to the forward and backward probabilities, respectively.
The entries of the matrices $F_{x}$ and $B_{x}$ can be efficiently calculated in an iterative manner,
\begin{eqnarray}
f_{x,r}(0) &=& \frac{1}{l\cdot l'^n},\quad r\in S,\\
f_{x,r}(i) &=& \sum_{s\in S} f_{x,s}(i-1)\cdot\theta_{r;x_i,s},\quad r\in S,\;1\leq i\leq n,
\end{eqnarray}
and
\begin{eqnarray}
b_{x,r}(n) &=& 1, \quad r\in S,\\
b_{x,r}(i) &=& \sum_{s\in S} \theta_{s;x_{i+1},r}\cdot b_{x,s}(i+1),\quad r\in S,\;0\leq i\leq n-1.
\end{eqnarray}

\begin{proposition}\label{p-SA-bw}
In view of the sufficient statistic $v$, we have for all $s,s'\in S$ and $a\in\Sigma$,
\begin{eqnarray}
v_{s';a,s} &=& \sum_{x\in\Sigma^n}
\frac{u_{x}}{p_{X|\Theta}(x|\theta)}
\sum_{i=1}^{n} f_{x,s}(i-1)\cdot \theta_{s';a,s}\cdot b_{x,s'}(i).
\end{eqnarray}
\end{proposition}
\begin{proof}
Let $I_A$ denote the indicator function of a proposition $A$; i.e., $I_A=1$ if $A$ is true and $I_A=0$ otherwise.
For each state sequence $\sigma\in S^{n+1}$, we have
\begin{eqnarray*}
v_{s';a,s} 
&=& \sum_{x\in\Sigma^n} \sum_{i=1}^{n} I_{(\sigma_i\sigma_{i+1} = ss')} \cdot I_{(x_i=a)}\cdot u_{x,\sigma}.
\end{eqnarray*}
Thus in view of~(\ref{eq-SA-Eu}), we obtain
\begin{eqnarray*}
v_{s';a,s} = \sum_{x\in\Sigma^n} \frac{u_{x}}{p_{X|\Theta} (x|\theta)} 
\sum_{i=1}^{n} \sum_{\sigma\in S^{n+1}} I_{(\sigma_i\sigma_{i+1} = ss')} \cdot I_{(x_i=a)} \cdot p_{X,S|\Theta}(x,\sigma|\theta).
\end{eqnarray*}
The innermost term is the sum of all probabilities of pairs $(x,\sigma)$ for an input sequence $x$ and all state sequences $\sigma$ such that $\sigma_i\sigma_{i+1}=ss'$ and $x_i=a$.
That is, observing the input sequence $x$ and a transition from state $s$ to state $s'$ at position $i$ with $x_i=a$.
Thus we have
\begin{eqnarray*}
\sum_{\sigma\in S^{n+1}} I_{(\sigma_i\sigma_{i+1} = ss')} \cdot I_{(x_i=a)} \cdot p_{X,S|\Theta}(x,\sigma|\theta) 
&=& f_{x,s}(i-1)\cdot \theta_{s';a,s}\cdot b_{x,s'}(i).
\end{eqnarray*}
The result follows.
\end{proof}

The proposition shows that the calculation of the forward and backward probability matrices yields directly the sufficient statistic without the need to estimate the 
counts $U=(u_{x,s})$.
This amounts to the Baum-Welch algorithm (Alg.~\ref{a-SA-bw}).
On the other hand, the EM algorithm requires to maintain the $l'^n\times l^{n+1}$ data set $U=(u_{x,s})$ from which the sufficient statistic can be established.
\begin{algorithm}
\caption{Baum-Welch algorithm for stochastic acceptors}\label{a-SA-bw}\index{Baum-Welch algorithm}
\begin{algorithmic}
\REQUIRE Stochastic acceptor $A= (S,\Sigma,P,\pi,f)$, $\pi$ uniform, joint probability function $p_{X,S|\Theta}$, 
parameter space $\Theta\subseteq\RR_{>0}^{l'l(l-1)}$, 
integer $n\geq 1$, observed data $u=(u_{x})\in\NN^{l'^n}$
\ENSURE Maximum likelihood estimate $\theta^*\in\Theta$
\STATE [Init] Threshold $\epsilon > 0$ and parameters $\theta\in\Theta$
\STATE [E-Step] Compute the sufficient statistic $v$ as in Prop.~\ref{p-SA-bw}
\STATE [M-Step] Compute solution $\hat \theta\in\Theta$ of the likelihood function $\ell_{X,S}$ using the sufficient statistic $v$ as in Prop.~\ref{p-SA-fo-l}
\STATE [Compare] If $\ell_{X}(\hat \theta) - \ell_{X}(\theta)>\epsilon$, set $\theta \leftarrow \hat\theta$ and resume with E-step
\STATE [Output] $\theta^* \leftarrow \hat\theta$
\end{algorithmic}
\end{algorithm}


\end{document}